%% file: Mahony_VIO_rss.tex
\documentclass{article}
\usepackage{arxiv}

\usepackage{natbib}
\usepackage{multicol}

\usepackage{url}
\usepackage{amsthm}
\usepackage[utf8]{inputenc}
\input{preamble_compatible}
\input{notation_defs}

\providecommand{\etal}{\textit{et al.}~}

\usepackage{caption}
\usepackage{subcaption}
\renewcommand{\mr}[1]{#1^\circ}

\usepackage{hyperref}
\newcommand{\VSLAM}{\mathbf{VSLAM}}
\newcommand{\MR}{\mathbf{MR}}
\newcommand{\vslam}{\mathfrak{vslam}}
\newcommand{\gothmr}{\mathfrak{mr}}

\begin{document}

\newcommand{\publicationdetails}
{van Goor, P., Mahony, R., Hamel, T., Trumpf, J. (2020). An Observer Design for Visual Simultaneous Localisation and Mapping with Output Equivariance. In \emph{Proceedings of 21st IFAC World Congress 2020}.}
\newcommand{\publicationversion}
{Author accepted version}


\title{An Observer Design for Visual Simultaneous Localisation and Mapping with Output Equivariance}
\headertitle{An Observer Design for Visual SLAM with Output Equivariance}

\author{
\href{https://orcid.org/0000-0003-4391-7014}{\includegraphics[scale=0.06]{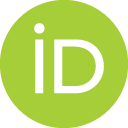}\hspace{1mm}
Pieter van Goor} \\
Department of Electrical, Energy and Materials Engineering\\
Australian National University\\
  ACT, 2601, Australia \\
\texttt{Pieter.vanGoor@anu.edu.au}
\\ \And 
\href{https://orcid.org/0000-0002-7803-2868}{\includegraphics[scale=0.06]{orcid.png}\hspace{1mm}
Robert Mahony} \\
Department of Electrical, Energy and Materials Engineering\\
Australian National University\\
  ACT, 2601, Australia \\
\texttt{Robert.Mahony@anu.edu.au}
\\ \And	
\href{https://orcid.org/0000-0002-7779-1264}{\includegraphics[scale=0.06]{orcid.png}\hspace{1mm}
Tarek Hamel} \\
I3S (University C\^ote d'Azur, CNRS, Sophia Antipolis)\\
and Insitut Universitaire de France\\
\texttt{THamel@i3s.unice.fr}
\\ \And	
\href{https://orcid.org/0000-0002-5881-1063}{\includegraphics[scale=0.06]{orcid.png}\hspace{1mm}
Jochen Trumpf} \\
Department of Electrical, Energy and Materials Engineering\\
Australian National University\\
  ACT, 2601, Australia \\
\texttt{Jochen.Trumpf@anu.edu.au}
}



\maketitle

\begin{abstract}
Visual Simultaneous Localisation and Mapping (VSLAM) is a key enabling technology for small embedded robotic systems such as aerial vehicles.
Recent advances in equivariant filter and observer design offer the potential of a new generation of highly robust algorithms with low memory and computation requirements for embedded system applications.
This paper studies observer design on the symmetry group proposed in \citep{2019_vangoor_cdc_vslam}, in the case where inverse depth measurements are available.
Exploiting this symmetry leads to a simple fully non-linear
gradient based observer with almost global asymptotic and local exponential stability properties.
Simulation experiments verify the observer design, and demonstrate that the proposed observer achieves similar accuracy to the widely used Extended Kalman Filter with significant gains in processing time (linear verses quadratic bounds with respect to number of landmarks) and qualitative improvements in robustness.
\end{abstract}



\section{Introduction}\label{sec:intro}

Visual Simultaneous Localisation and Mapping (VSLAM) and the closely related Visual Odometry (VO) are established topics in the robotics community \citep{2012_Kaess_IJRR,2015_Leutenegger_ijrr,2016_Faessler_JFR,2017_ForsterScaramuzza_tro,2017_ForsterDelaert_tro}.
They are key components of almost all aerial robotic systems \citep{2018_Delmerico_icra} and are used in a host of other robotic applications \citep{2008_Bonin-Font_JIRS} including autonomous driving and underwater robotics.
VSLAM is used to refer to the case of the general SLAM problem where the available measurements are the bearings of landmarks such as provided by image features obtained using a monocular camera.
Visual Odometry is a variant of the VSLAM problem where the solution is optimised for local consistency of the localisation of the system and updates landmark states only for currently visible landmarks.
While the VO community has focused on embedded systems applications, and places a premium on algorithms with low computational and memory requirements \citep{2018_Delmerico_icra}, the VSLAM community has placed a premium on large scale map building, loop closure and accuracy \citep{2016_Stachniss_Handbook,2016_Cadena_TRO}.
As a consequence, many state-of-the-art VO systems use filter based formulations \citep{2007_Mourikis_icra,2015_Bloesch_iros,2017_ForsterScaramuzza_tro,2014_Lynen_iros} in contrast to the full trajectory smoothing and graph based optimization formulation accepted as the community standard for SLAM problems \citep{2016_Cadena_TRO}.
Well engineered trajectory smoothing algorithms using short sliding-windows are still highly competitive algorithms for VO
\citep{2012_Kaess_IJRR,2015_Leutenegger_ijrr,2017_Qin_arxive,2017_ForsterDelaert_tro}.

The non-linear observer community has become interested in the visual SLAM and VO problem in the last few years.
Work by Guerrerio \etal \citep{2013_Guerreiro_TRO} and Louren{\c{c}}o \etal \citep{2016_LouGueBatOliSil} propose a non-linear observer for the ``robo-centric'' SLAM problem.
Recent work by Barrau \etal \citep{2016_Barrau_arxive,2017_Barrau_tac} introduce a symmetry group $\SE_{n+1}(3)$ for the SLAM problem and use this to derive an Invariant Kalman Filter algorithm that overcomes consistency issues that have plagued the EKF algorithms from the classical SLAM era \citep{2011_Dissanayake_ICIIS,2017_Zhang_ral}.
Parallel work by Mahony \etal 
\citep{2017_Mahony_cdc} show that this symmetry acts transitively on a principle fibre bundle $\calM_n(3)$ which forms a natural geometric state-space for the SLAM problem, overcoming the gauge uncertainty present in the usual pose-map state representation.
However, the symmetry induced by the group $\SE_{n+1}(3)$ applies only to the SLAM configuration state and is not compatible with bearing measurements.
As a consequence, applying the $\SE_{n+1}(3)$ symmetry to the visual SLAM problem still requires linearisation of the output map.
A new symmetry for the VSLAM problem was proposed in \citep{2019_vangoor_cdc_vslam} along with a non-linear observer. 
However, in this prior work the observer is closely based on \citep{2016_Hamel_cdc} and is derived in local coordinates and then lifted onto the symmetry group. 
It is of interest to consider the case where the observer is designed  explicitly using the symmetry structure.


In this paper, we present a highly robust, simple, and computationally cheap non-linear observer for the visual SLAM problem based on the new symmetry of the SLAM configuration space, first presented in \citep{2019_vangoor_cdc_vslam}.
The symmetry is associated with the novel $\VSLAM_n(3)$ Lie-group, which acts on the raw pose-map configuration coordinates and is compatible with the SLAM configuration manifold $\calM_n(3)$ \citep{2017_Mahony_cdc}.
The symmetry group structure introduced allows direct application of previous work by Mahony \etal \citep{RM_2013_Mahony_nolcos} in development of non-linear observers to yield a novel observer for continuous-time VSLAM.
In the design of this observer, it is assumed measurements of the inverse depths of landmarks are available in addition to the bearing measurements.
In practice, the inverse depth may be measured by using optical flow, triangulation, or depth cameras.
The resulting algorithm is fully non-linear; no linearisation is required of the system or output maps.
The approach has the advantage that constant gains can be used in the filter (no Riccati gains need be computed on-line) leading to lower computation and memory requirements.
This additionally leads to a reduction in the number of parameters that need to be tuned in comparison with a standard EKF, making the proposed filter simpler to use in practice.
The inherent symmetry of the approach ensures high levels of robustness and Theorem \ref{th:observer} proves almost global asymptotic and local exponential stability of the error coordinates.
The convergence properties of the filter are demonstrated through a simulation experiment.
Additional simulation experiments compare an EKF with our observer.
These show that our observer achieves comparable mean RMSE to the EKF and has fewer outliers, and operates with a computational complexity that is only linear in the number of landmarks compared to quadratic complexity for the EKF. 

\section{Notation}\label{sec:notation}


The special orthogonal group is the set of rotation matrices and is denoted $\SO(3)$ with Lie algebra $\so(3)$.
The special Euclidean group is the set of rigid body transformations and is denoted $\SE(3)$ with Lie algebra $\se(3)$.
The group of positive real numbers equipped with multiplication is denoted $\MR$ with Lie algebra $\gothmr$.
We use the notation $R_P \in \SO(3)$ and $x_P \in \R^3$ to denote the rotation and translation components of a rigid-body transformation $P \in \SE(3)$ and write
\begin{equation}
P = \begin{pmatrix}
R_P & x_P \\
0 & 1 \end{pmatrix}.
\label{eq:matrix_pose_SE3}
\end{equation}

The pose of a vehicle moving in Euclidean space is written $P \in \SE(3)$.
The kinematics of such a pose frame are written as
\begin{gather}
\dot{P} = PU, \quad \dot{R}_P = R_P \Omega_U^\times, \quad \dot{x}_P = R_P V_U \label{eq:rigid_kinematics_coords}
\end{gather}
where $\Omega_U=(\Omega_1,\Omega_2,\Omega_3)^\top$ and $V_U$ are the body-fixed rotational and translational velocity vectors, respectively, and
\begin{equation}
U = (\Omega_U^\times,V_u)^{\wedge}
:=
\begin{pmatrix}
  \Omega_U^\times & V_U \\ 0 & 0
\end{pmatrix}, \quad \Omega^\times_U=
\begin{pmatrix} 0 & -\Omega_3 & \Omega_2 \\
\Omega_3 & 0 & -\Omega_1 \\
-\Omega_2 & \Omega_1 & 0
\end{pmatrix}.
\label{eq:wedge}
\end{equation}
One has that $\Omega^\times_U w = \Omega_U \times w$ for any $w \in \R^3$ where $\times$ refers to the vector (cross) product.

For a unit vector $y \in \Sph^2 \subset \R^3$, the projector is given by
\begin{align} \label{eq:projector}
  \Pi_y := I_3 - yy^\top,
\end{align}
and has the property $\Pi_y = - y^\times y^\times$.



\section{Problem formulation}\label{sec:problem_formulation}

The \emph{total space} coordinates for the SLAM problem are defined with respect to an unknown fixed but arbitrary reference $\frameZero$.
Let $P \in \SE(3)$ represent the body-fixed frame coordinates of the robot with respect to this reference frame.
Let
\[
p_i \in \R^3, \quad\quad i = 1, \ldots, n,
\]
be sparse points in the environment expressed with respect to the reference frame $\frameZero$.
The \emph{total space} of the SLAM problem is the product space
\begin{align}
  \totT_n(3) = \SE(3) \times \R^3 \times \cdots \times \R^3
  \label{eq:tot_T}
\end{align}
made up of these raw coordinates $\Xi := (P, p_1, \ldots, p_n)$.
The bearing of a point $p_i$ co-located with the robot pose centre $x_P$ is undefined, so the VSLAM problem can only be considered on the \emph{reduced total space}
\begin{align}
  \mr{\totT}_n(3) = \left\{ (P, p_i) \in \totT_n(3) \; \vline \; (\forall i) \; p_i \neq x_P \right\}.
  \label{eq:reduced_total_space}
\end{align}
Moreover, since all the measurements of the VSLAM problem considered are made in the body-fixed frame the solution is only well defined up to an $\SE(3)$ gauge transformation \citep{2001_Kanatani_TIT}.
This property can be expressed as an invariance of the problem formulation and leads to the quotient structure of the SLAM manifold proposed in \citep{2017_Mahony_cdc}.
To keep the derivation simple and more accessible, in the present paper we will define the group actions and derive the observer on the reduced total space.

The measurements considered are spherical coordinates $y_i$ of body-fixed frame observations of points in the environment, which in practice may be obtained from a calibrated monocular camera.
Additionally, in this analysis, we assume inverse depth estimates $z_i$ are also available.
That is, for a given robot pose $P$ and environment point $p_i$,
\begin{subequations}\label{eq:h_output}
  \begin{align}
y_i & := \frac{R_P^\top(p_i-x_P)}{|p_i-x_P|},  \label{eq:h_output_y} \\
z_i & :=  |p_i - x_P|^{-1}.  \label{eq:h_output_z}
\end{align}
\end{subequations}
The combined output space is $\calN_n(3) = (S^2 \times \R^+) \times \cdots \times (S^2 \times \R^+)$ and we write $(y, z) = h(\Xi) = (h_1(\Xi),\dots,h_n(\Xi)) = ((y_1, z_1), \dots, (y_n, z_n))$, where appropriate.

Let $U \in \se(3)$ denote the velocity of the robot.
Assume that the environment points $p_i$ being observed are static, and thus do not have a velocity.
The tangent space of $\totT_n(3)$ at a point $\Xi = (P, p_1, \ldots, p_n)$ can be identified with the matrix subspace
\[
  (PU, u_1,...,u_n), \quad U \in \se(3), \quad u_1,...,u_n \in \R^3.
\]
The system kinematics can then be written as
\begin{equation}
\ddt ( P , p_1, \ldots, p_n )  = ( P U, 0, \ldots, 0 ).
\label{eq:tot_kinematics}
\end{equation}

We assume that the robot velocity $U \in \se(3)$ is measured. 
We will also measure the optical flow of each landmark
\[
\phi_i \in T_{y_i} \Sph^2 \subset \R^3
\]
by numerically differentiating the coordinates of $y_i$ between consecutive measurements.
Here, we express $\phi_i \in T_{y_i} \Sph^2$ using the coordinates obtained by embedding $\Sph^2 \subset \R^3$.
Define a measurement velocity space
\begin{align}
\vecV = \se(3) \times \R^3 \times \cdots \times \R^3
\label{eq:vecV}
\end{align}
with elements $(U, \phi_1, \ldots, \phi_n)$.



\section{Symmetry of the VSLAM problem}\label{sec:symmetry}

The symmetry group of the VSLAM for $n$ landmarks in Euclidean 3-space with separate bearing and range measurements problem is the \emph{visual SLAM group} $\VSLAM_n(3)$ first described in \citep{2019_vangoor_cdc_vslam}.
However, in this paper the VSLAM group and its actions are presented in a different form to \citep{2019_vangoor_cdc_vslam}.
%
%
%

In the following, we will write
$(Q_A,a)_1$ instead of the more formal $({\left(Q_A\right)}_1,a_1)$ and sometimes write $(Q_A,a)_i$ to represent the tuple $(Q_A,a)_1,\dots,(Q_A,a)_n$.
Similarly, we will sometimes write $(P,p_i)$ instead of
$(P,p_1,\dots,p_n)$.

The VSLAM group \citep{2019_vangoor_cdc_vslam} may be written
\begin{multline*}
\VSLAM_n(3) = \{ (A,(Q_A,a)_1, \ldots,(Q_A,a)_n) \;|\; \\
 A \in \SE(3),\; {\left(Q_A\right)}_i \in \SO(3),\; a_i\in \MR, \; i = 1, \ldots, n\}.
\end{multline*}

\begin{lemma}
The set $\VSLAM_n(3)$ is a Lie group, defined as
$$ \VSLAM_n(3) := \SE(3) \times (\SO(3)\times \MR)^n.$$
\end{lemma}

The visual SLAM group acts as a symmetry group on the reduced total space $\mr{\totT}_n(3)$.

\begin{lemma}\label{lem:Upsilon_action}
The mapping $\Upsilon : \VSLAM_n(3) \times \mr{\totT}_n(3) \rightarrow \mr{\totT}_n(3)$ defined by
\begin{align}
\Upsilon((A,&(Q_A,a)_i),(P,p_i)) \notag \\
&= (PA,(a^{-1}R_{PA}Q_A^\top R_P^\top(p-x_P) +x_{PA})_i), \label{eq:Upsilon}
\end{align}
is a right group action of $\VSLAM_n(3)$ on $\mr{\totT}_n(3)$.
\end{lemma}

The group action for the robot pose is rigid-body transformation.
The group action for environment points is considerably more subtle and can be understood conceptually as a sequence of operations:  firstly, the reference frame coordinates of an environment point are written in the body-fixed frame, this point is then rotated by $Q_A^\top$ and then scaled by $a^{-1}$,
before these body-fixed frame coordinates are rewritten in the inertial frame using the \emph{new body-fixed frame} reference.

\begin{figure}[ht]
\begin{centering}
\includegraphics[width = 0.8\linewidth]{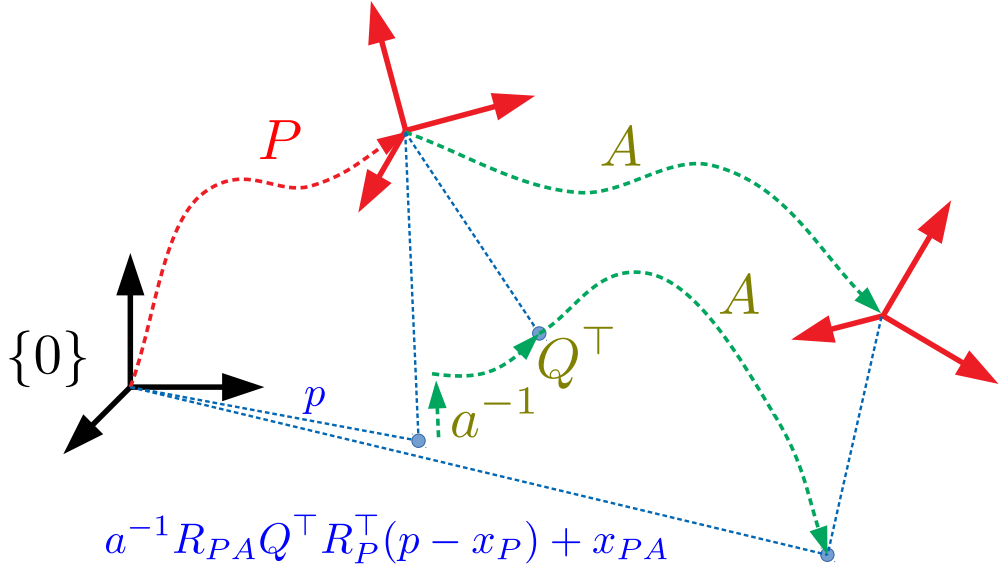}
\caption{Group action $\Upsilon((A,(Q_A,a)_i),(P,p_1, \ldots, p_n))$.
The pose $P \mapsto P A$, that is the tip point of the pose is updated by the correction $A \in \SE(3)$.
The body fixed-frame environment points $R_A^\top (p_i - x_P)$ are rotated by $Q_A^\top$ and scaled by $a^{-1}$ in the body-fixed frame before \emph{transforming with the robot pose} to a new point $p_i'$ which is rewritten in the inertial frame.
}
\label{fig:symmetry_action}
\end{centering}
\end{figure}

%

A key property of the proposed structure is that there is a compatible group operation on the output $\calN_n(3)$ of the system.

\begin{lemma}\label{lem:rho}
The action $\rho : \VSLAM_n(3) \times \calN_n(3) \rightarrow \calN_n(3)$ defined by
\begin{align}\label{eq:rho}
\rho((A,(Q_A,a)_i),(y, z)_i)& := ((Q_A^\top y, az)_i)
\end{align}
is a transitive right action on $\calN_n(3)$.
Furthermore, one has
\begin{align*}
\rho((A,(Q_A,a)_i), h(\Xi)) & = h(\Upsilon((A, (Q_A,a)_i),\Xi))
\end{align*}
where $h$ is given by \eqref{eq:h_output}.
That is, $h$ is equivariant with respect to the actions $\Upsilon$ and $\rho$.
\end{lemma}

%

\subsection{Lift of the SLAM kinematics}

A key aspect of the proposed approach is that the symmetry group $\VSLAM_n(3)$ and the reduced total space $\mr{\totT}_n(3)$ are quite different spaces.
The difference is particularly clear in studying the structure of the lifted kinematics on the $\VSLAM_n(3)$ group.

The Lie-algebra of $\VSLAM_n(3)$ is
\begin{equation*}
\vslam_n(3) = \se(3) \times (\so(3) \times \gothmr) \times \dots \times (\so(3) \times \gothmr).
\end{equation*}
We write $(U,(W,w)_i) \in \vslam_n(3)$ with $U \in \se(3)$, $W_i \in \so(3)$ and $w_i \in \gothr$, where
\begin{align*}
U=\begin{pmatrix} \Omega_U^\times & V_U \\ 0 & 0 \end{pmatrix}.
\end{align*}

In order to implement an observer on the VSLAM group, it is necessary to lift the velocity measurements $(U,\phi_i) \in \vecV$ \eqref{eq:vecV} to elements $(U,(W,u)_i) \in \vslam_n(3)$ such that the resulting group velocity is compatible with the system kinematics.
That is, an algebraic map $\lambda : \mr{\totT}_n(3) \times \vecV \rightarrow \vslam_n(3)$ is required such that
\begin{align} \label{eq:lift-condition}
\td \Upsilon_{(P,p_i)} \lambda ((P,p_i),(U,\phi_i)) = (PU, 0, \ldots, 0).
\end{align}

\begin{proposition}
The map $\lambda :\mr{\totT}_n(3) \times \vecV \rightarrow \vslam_n(3)$ defined by
\begin{align} \label{eq:lambda}
\lambda((P,p_i),(U,\dot{y}_i)) := \left(U ,\left((\phi \times y)^\times, z y^\top V_U\right)_i \right),
\end{align}
where $(y,z)_i = h((P,p_i))$ satisfies the lift condition \eqref{eq:lift-condition}.
\end{proposition}

\begin{proof}
Under the static landmark assumption $p_i = 0$, the optic flow $\phi_i = \dot{y_i}$ is given by
\begin{align} \label{eq:optical-flow}
\phi_i &= -\Omega_U^\times y_i - z_i (I-y_i y_i^\top) V_U.
\end{align}
Let $(\lambda_A, (\lambda_Q^\times, \lambda_a)_i) := \lambda((P,p_i),(U,\phi_i))$.
Evaluating the left-hand side of \eqref{eq:lift-condition}, one has
\begin{align*}
\td &\Upsilon_{(P,p_i)} \lambda ((P,p_i),(U,\phi_i)) \\
&= \tD \Upsilon_{(P,p_i)}(\id) [ (\lambda_A, (\lambda_Q^\times, \lambda_a)_i) ], \\
&= (P\lambda_A, (- \lambda_a (p-x_P) + R_P \Omega_{\lambda_A}^\times R_P^\top (p - x_P) \\
&\hspace{2cm} - R_P \lambda_Q^\times R_P^\top (p - x_P) + R_P V_U )_i).
\end{align*}
This expression may be written in terms of $(y,z)_i$ as follows.
\begin{align*}
\td &\Upsilon_{(P,p_i)} \lambda ((P,p_i),(U,\phi_i)) \\
&= (P\lambda_P, (- \lambda_a R_P z^{-1} y + R_P (\Omega_{\lambda_A}^\times - \lambda_Q^\times) z^{-1} y + R_P V_U )_i).
\end{align*}
Multiply the landmark velocity terms by $R_P^\top$ and substitute in the values for $\lambda$ to obtain
\begin{align*}
- \lambda_a & z^{-1} y + (\Omega_{\lambda_A}^\times - \lambda_Q^\times) z^{-1} y + V_U \\
&= -yy^\top V + z^{-1} (\Omega_U - (\phi \times y)^\times) y + V_U, \\
&= z^{-1} \Omega_U^\times y - z^{-1} y^\times y^\times \phi + (I-yy^\top) V , \\
&= z^{-1} \Omega_U^\times y - z^{-1} (\Omega_U^\times y_i + z_i (I-y_i y_i^\top) V_U) + (I-yy^\top) V , \\
&= 0.
\end{align*}
Hence $\td \Upsilon_{(P,p_i)} \lambda ((P,p_i),(U,\phi_i)) = (P\lambda_A, R_P 0) = (PU, 0)$ as required.
\end{proof}

The lifted velocity $(U,(W,u)_i) = \lambda((P,p_i), (U, \phi_i))$ induces kinematics on the symmetry group that project down to the state space trajectory.
Since the group action is not free, the stabiliser of $\Upsilon$ is non-trivial, and there are directions in $\vslam_n(3)$, in particular $y_i^\times \in \so(3)$, that are not constrained by the lift requirement \eqref{eq:lift-condition}.
The lift in direction $y_i^\times \in \so(3)$ is chosen to be zero without loss of generality.

The lift $\lambda$ will enable us to go on and apply the observer design methodology developed in \citep{RM_2013_Mahony_nolcos}. 

\section{Observer design}\label{sec:observer}

We approach the observer design by considering the lifted kinematics of the system on the symmetry group and designing the observer on $\VSLAM_n(3)$.
Let $\Xi(t) = (P(t),p_i(t)) \in \totT_n(3) $ be the `true' configuration of the SLAM problem, noting that $\Xi(t)$ is defined relative to some arbitrary reference $\frameZero$.
Let $X(t) = (A(t),(Q(t) ,a(t))_i) \in \VSLAM_n(3)$ and define the \emph{lifted kinematics}  \citep{RM_2013_Mahony_nolcos}
\begin{align}
\ddt X(t)&= X \lambda(\Upsilon(X, \Xi(0)), (U, \phi_i)) \notag \\
\quad X(0) &= \id.
\label{eq:ddtX}
\end{align}
Equation \eqref{eq:ddtX} evolves on the VSLAM group where $(U, \phi_i) \in \vecV$ are the measured velocities and $\lambda$ is the lift function \eqref{eq:lambda}.

Choose an arbitrary origin configuration
\[
\mr{\Xi}= (\mr{P}, \mr{p}_i) \in \totT_n(3).
\]
If the initial condition $X(0) \in \VSLAM_n(3)$ of the lifted kinematics satisfies $\Upsilon(X(0),\mr{\Xi}) = \Xi(0)$ then \eqref{eq:ddtX} induces a trajectory that satisfies
\[
\Xi(t) = \Upsilon( X(t) , \mr{\Xi}) \in \totT_n(3)
\]
for all time \citep{RM_2013_Mahony_nolcos}.

Let the observer be defined as
\[
\hat{X} = (\hat{A},(\hat{Q},\hat{a})_i) \in \VSLAM_n(3).
\]
The lifted kinematics \eqref{eq:ddtX} provide the internal model for the observer design.
That is, the kinematics of the observer are given by
\begin{align}
\ddt \hat{X}(t)&= \hat{X} \lambda(\Upsilon(X, \mr{\Xi}), (U, \phi_i))  - \Delta \hat{X} \notag, \\
\hat{X}(0) &= \id,
\label{eq:ObserXi}
\end{align}
where $\Delta \in \vslam_n(3)$ is an innovation term to be assigned.
Note that $\lambda(\Upsilon(X, \mr{\Xi}), (U, \phi_i))$ is shown in \eqref{eq:lambda} to depend only on the measured quantities $y_i, z_i, U, \phi_i$, and therefore can be implemented in the observer kinematics \eqref{eq:ObserXi}.
The configuration estimate generated by the observer is given by
\[
\hat{\Xi} = (\hat{P}, \hat{p}_i)
= \Upsilon( \hat{X}, \mr{\Xi}) \in \mr{\totT}_n(3)
\]
given the reference $\mr{\Xi} \in \totT_n(3)$.

\begin{theorem}\label{th:observer}
Consider the kinematics \eqref{eq:ddtX} evolving on $\VSLAM_n(3)$ along with bounded outputs ${\bf y}=\{(y,z)_i\}=h(\Xi(t)) \in \calN_n(3)$ given by \eqref{eq:h_output}.
Fix an arbitrary origin configuration $\mr{\Xi}= (\mr{P}, \mr{p}_i)$ and define the output error ${\bf e}=\{(e_y,e_z)_i\}$ as
\begin{align}
{\bf e}=\rho(\hat{X}^{-1},\{(y,z)_i\})=\rho(E^{-1}, \{(\mr{y},\mr{z})_i\}) \in \calN_n(3). \label{e}
\end{align}
where $\{(\mr{y},\mr{z})_i\} = h(\mr{\Xi})$ and $E=\hat{X} X^{-1}$.
Consider the observer defined in \eqref{eq:ObserXi} and choose the
innovation term $\Delta := (\Delta_A, (\Delta_Q, \Delta_a)^i)$ as follows:
\begin{align}
\Delta_Q^i & := -k_{Q_i}( e_{y_i} \times \mr{y}_i)^\times \label{eq:Delta_Q}\\
\Delta_a^i & := -k_{a_i} \frac{e_{z_i} - \mr{z}_i }{e_{z_i}} \label{eq:Delta_a} \\
\Delta_A &:= -k_{A} \Ad_{\hat{A}} \left( \begin{matrix}
\Omega_\Delta^\times & V_\Delta \\ 0 & 0
\end{matrix} \right), \label{eq:Delta_A} \\
\begin{pmatrix} \Omega_{\Delta} \\ V_\Delta \end{pmatrix}
&:= \left( \sum_{i=1}^n \begin{pmatrix} \Pi_{\hat{y}_i} & \hat{z}_i \hat{y}_i^\times \\ -\hat{z}_i \hat{y}_i^\times & \hat{z}_i^2 \Pi_{\hat{y}_i} \end{pmatrix} \right)^{-1}
  \left( \sum_{i=1}^n \begin{pmatrix}
    -\hat{y}_i^\times \phi_i \\ -\hat{z}_i \phi_i
  \end{pmatrix} \right) - \begin{pmatrix}
    \Omega_U \\ V_U
  \end{pmatrix}, \notag 
\end{align}
where the gains $k_{a_i}, k_{Q_i}$ and $k_{A}$ are positive scalars (for $i = 1, \ldots, n$), and the matrix inverse in the definition of $\Delta_{A}$ is assumed to be well-defined.
Then the configuration estimate $\hat{\Xi}(t) = \Upsilon(\hat{X}(t),\mr{\Xi})$ converges almost globally asymptotically and locally exponentially to the true state $\Xi(t)=\Upsilon(X(t),\mr{\Xi})$ up to a possibly time-varying element in $\SE(3)$.
\end{theorem}

\begin{proof}
Let $X(t) \in \VSLAM_n(3)$ satisfy the lifted kinematics \eqref{eq:ddtX} with $\Upsilon(X(0),\mr{\Xi}) = \Xi(0)$.
It follows that $\Xi(t)=\Upsilon(X(t),\mr{\Xi})$ \citep{RM_2013_Mahony_nolcos}.
Define
\begin{equation}
E=\hat{X} X^{-1} =: (\tilde{A}, (\tilde{Q},\tilde{a})_i) \in \VSLAM_n(3),
\label{eq:E}
\end{equation}
with $\tilde{A}:=\hat{A}A^{-1}$,  $(\tilde{Q},\tilde{a})_i=\left( \hat{Q} Q^\top , \hat{a} a^{-1} \right)_i$.
Using \eqref{eq:ddtX} and \eqref{eq:ObserXi}, it is straightforward to verify that
\begin{equation}
\dot E = (-\Delta_A \tilde{A}, (-\Delta_Q \tilde{Q}, -\Delta_a \tilde{a})_i).
\label{dotE}
\end{equation}

Using the fact that $E^{-1} = (\tilde{A}^{-1},(\tilde{Q}^\top, \tilde{a}^{-1})_i)$  then each element of equation \eqref{e} becomes
\begin{align}
(e_{y_i}, e_{z_i})&=\left (\tilde{Q} \mr{y}_i,\tilde{a}^{-1} \mr{z}_i \right)\label{eq:e}.
\end{align}
Based on \eqref{dotE}, the error kinematics satisfy
\begin{align} \label{eq:error_kinematics}
(\dot{e}_{y_i}, \dot{e}_{z_i})&= \left (-\Delta_Q^i e_{y_i}, \Delta_a^i e_{z_i}\right ).
\end{align}

We first prove almost-global asymptotic and local exponential stability of the equilibrium $(e_{y_i}, e_{z_i})=(\mr{y}_i, \mr{z}_i)$ for the error kinematics \eqref{eq:error_kinematics}.
Consider the following candidate (positive definite) Lyapunov function $\Lyap: \calN_n(3) \to \R^+$,
\begin{equation}
{\cal L}=\frac{1}{2}\sum_{i=1}^{n} \left (\left|e_{y_i}-\mr{y}_i \right|^2+\left (e_{z_i}-\mr{z}_i\right)^2 \right). 
\label{eq:lyap}
\end{equation}
Differentiating  ${\cal L}$ and using \eqref{eq:Delta_Q} and \eqref{eq:Delta_a}, one gets:
\begin{align*}
\dot{\cal L} &=\sum_{i=1}^{n} \left (\left (e_{y_i}-\mr{y}_i \right )^\top \dot{e}_{y_i}+\left (e_{z_i}-\mr{z}_i\right)\dot{e}_{z_i} \right ) , \\
&= \sum_{i=1}^{n} \left( - \left( e_{y_i}-\mr{y}_i \right)^\top \Delta_Q^i e_{y_i} + \left( e_{z_i}-\mr{z}_i \right) \Delta_a^i e_{z_i} \right ) , \\
&=-\sum_{i=1}^{n} \left ( k_{y_i} \left | e_{y_i} \times \mr{y}_i \right |^2 + k_{z_i} \left (e_{z_i}-\mr{z}_i\right)^2 \right ).
\end{align*}
The time derivative of the Lyapunov function is negative definite and equal to zero when $e_{y_i}=\pm\mr{y}_i$ and $e_{z_i}=\mr{z}_i$.
Direct application of Lyapunov's theorem ensures that the equilibrium $(e_{y_i},e_{z_i})=(\mr{y}_i, \mr{z}_i)$ is almost-globally \emph{asymptotically} stable\footnote{It is straightforward to verify that the equilibrium point $e_{y_i}=-\mr{y}_i$ is unstable.}.

To prove local exponential stability of the observer it suffices to split the Lyapunov function into two parts
\[
{\cal L}= {\cal L}_y +{\cal L}_z
\]
${\cal L}_z =\frac{1}{2}\sum_{i=1}^{n} \left (e_{z_i}-\mr{z}_i\right)^2$
and verify that $\dot{\cal L}_z \leq -2 \min(k_{z_i}) {\cal L}_z$, with ${\cal L}_z$ converging exponentially to zero.
Consider ${\cal L}_y=\sum_{i=1}^{n} {\cal L}_{y_i}$ with ${\cal L}_{y_i}=\frac{1}{2} \left|e_{y_i}-\mr{y}_i \right|^2$.
If there exists a positive number $\epsilon$ such that ${\cal L}_{y_i} \leq 2 - \epsilon$, that is $e_{y_i}(0)$ is not in the opposite direction of $\mr{y}(0)$, for all $i=\{1,\ldots,n\}$, then
\[
\dot{\cal L}_{y_i} \leq -2 \min(k_{y_i}) \epsilon {\cal L}_{y_i}
\]
This demonstrates local exponential stability (in a large domain) of the equilibrium
$(e_{y_i},e_{z_i})=(\mr{y}_i, \mr{z}_i)$.

In the limit, at the stable equilibrium point $(e_{y_i},e_{z_i})=(\mr{y}_i, \mr{z}_i)$), \eqref{eq:e} implies that
\begin{align*}
(\hat{y},\hat{z})_i =
\left(Q_{\hat{A}}^\top \mr{y},\hat{a} \mr{z} \right)_i =
\left(Q_{\hat{A}}^\top e_{y},\hat{a} e_{z} \right)_i =
(y, z)_i,
\end{align*}
for all $i=1,\ldots,n$.
This in turn implies
\begin{align*}
\rho(\hat{X}, (\mr{y},\mr{z})_i) & = h(\Upsilon(\hat{X}, \mr{\Xi})) = h ( \Upsilon(X, \mr{\Xi})) \\
& = \rho(X, (\mr{y}, \mr{z})_i)
\end{align*}
Regarding just the central equality, and noting that $h$ only preserves the relative pose on $\mr{\totT}_n(3)$, it follows that
\[
\Xi = \Upsilon(X, \mr{\Xi}) = \Upsilon(\hat{X}, \mr{\Xi}) \cong \hat{\Xi}
\]
The symbol $\cong$ indicates that $\Xi = \hat{\Xi}$ up to the possibly time-varying gauge transformation $\tilde{A}$.
That is
$$(\hat{P},\hat{p}_i) = (\tilde{A}^{-1} P,  \tilde{A}^{-1} (p_i)) = (\tilde{A}^{-1} P,  R_{\tilde{A}}^\top(p_i - x_{\tilde{A}})).$$
This concludes the proof of the almost-global asymptotic and local exponential stability.

\end{proof}

\begin{remark}
Observe that the output error ${\bf e}$ is independent of the $\SE(3)$ innovation $\Delta_A$ and the primary stability analysis in Theorem \ref{th:observer} is undertaken on the output space, not the state-space.
This is a key property of the $\VSLAM_n(3)$ symmetry and is intrinsic to the invariance of the underlying SLAM problem discussed in Section \S \ref{sec:problem_formulation}.
The particular choice of innovation $\Delta_A$ in \eqref{eq:Delta_A} minimizes the least squares drift in the visual odometry error as observed from direct measurements of landmark coordinates and optical flow.
This is only one of a family of possible choices (for example, \cite{2017_Mahony_cdc}), however, further analysis of this question is beyond the scope of the present paper.
\end{remark}

\section{Simulation Results}

The first simulation experiment was conducted to verify the observer design in Theorem \ref{th:observer}.
A robot is simulated to move in a circle with velocity $V_U = (0.1,0,0)$ m/s, $\Omega_U = (0,0,0.02\pi)$ rad/s on the ground, with 10 landmarks uniformly distributed in a 0.5-1 m band around the robot's path.
The reference configuration $\mr{\xi}$ of the observer is randomly set, and the observer $\hat{X}$ is initialised to the identity group element.
All landmarks are assumed to be measured at all times, and no noise is added to the system.
The gains of the observer are set to $k_{Q_i} = 0.05$, $k_{a_i} = 0.02$, $k_{A} = 0.03$ for all $i=1,...,10$.
The observer equations are implemented with Euler integration using a time step of 0.5 s.
Figure \ref{fig:simple_sim} shows the evolution of the Lyapunov function \eqref{eq:lyap} components for each landmark over 100 s.
The bearing storage refers to the component $l_y^i := \frac{1}{2}\vert e_{y_i} - \mr{y}_i \vert^2$ and the inverse depth storage refers to the component $l_z^i := \frac{1}{2}\vert e_{z_i} - \mr{z}_i \vert^2$ for each landmark index $i$.
The top two plots show the value of these functions for each landmark, and the bottom two plots show the log value for each landmark.
The plots clearly show the almost-global asymptotic and local exponential convergence of the observer's error system.

\begin{figure}[!htb]
  \centering
  \includegraphics[width=0.85\linewidth]{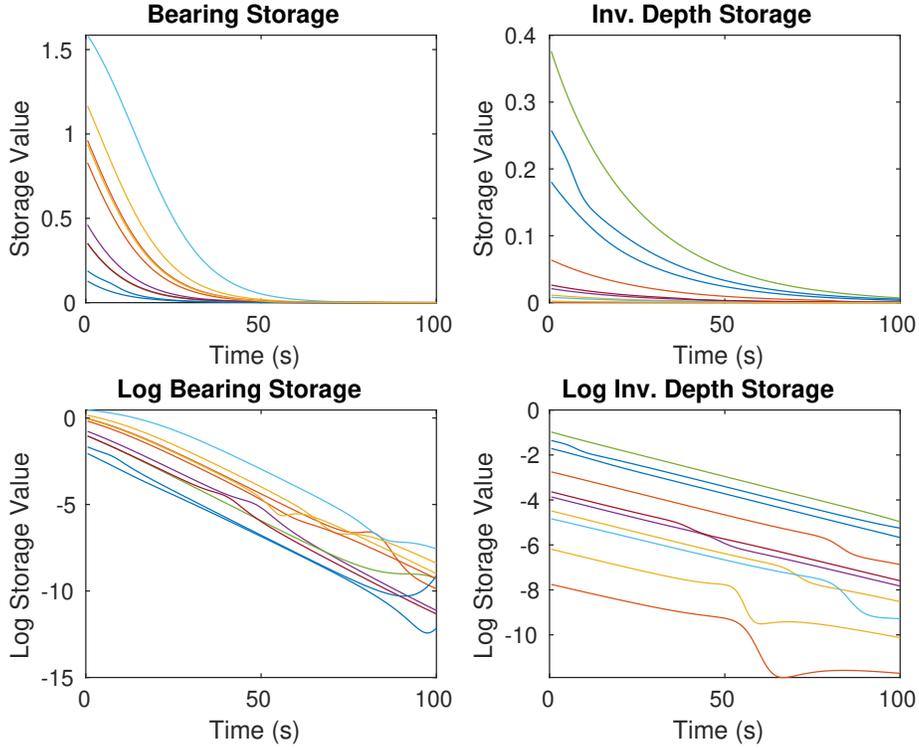}
  \caption{The evolution of the individual components of the Lyapunov function \eqref{eq:lyap} for 10 landmarks over time.}
  \label{fig:simple_sim}
\end{figure}

Additional simulations were carried out to compare the non-linear observer proposed in Theorem \ref{th:observer} with an Extended Kalman Filter (EKF).
A robot is simulated to move in a circle with velocity $V_U = (0.1,0,0)$ m/s, $\Omega_U = (0,0,0.02\pi)$ rad/s on the ground, with $n$ landmarks uniformly distributed in a 0.5-1 m band around the robot's path.
The robot is modelled to have a sensor range of 1 m.
The reference configuration $\mr{\xi}$ is initialised without any landmarks, and the observer group element is initialised to identity.
When landmarks are first seen, their inertial frame position is computed using the observer's current position estimate, and the reference configuration is augmented with this value.
When landmarks are not within the sensing range, the observer equations cannot be used, and the current observer estimate of the landmark position is fixed until the landmark is next seen.
All noise added to the input velocities and output measurements is drawn from zero-mean Gaussian distributions.
The linear velocity noise has variance $0.2$, the angular velocity noise has variance $0.1$, the optical flow noise has variance $0.02$, the bearing measurement noise has variance $0.01$, and the inverse depth measurement noise has variance $0.4$.
The EKF is implemented with the system equation $\eqref{eq:tot_kinematics}$, and the measurement equations $\eqref{eq:h_output}$.
The gains of the observer are set to $k_{Q_i} = 0.25$, $k_{a_i} = 0.1$, $k_{A} = 0.1$ for all $i=1,...,n$, and the observer equations were implemented using Euler integration with a time step of $0.5$ s.

\begin{figure}[!htb]
  \centering
  \begin{subfigure}[t]{0.45\linewidth}
    \centering
    \includegraphics[width=\textwidth]{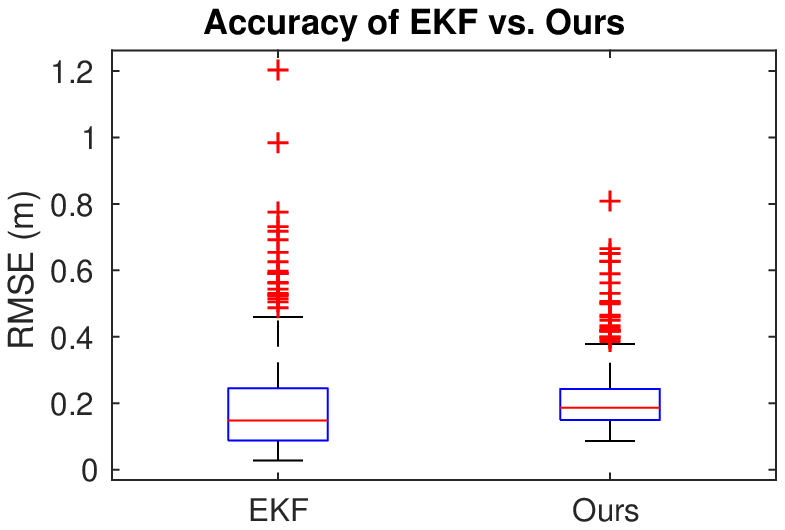}
    \caption{Boxplot of the RMSE of EKF and our observer on 50 landmarks over 500 trials.}
    \label{fig:sim_rmse_comparison}
  \end{subfigure}
  \hspace{0.05\linewidth}
  \begin{subfigure}[t]{0.45\linewidth}
    \centering
    \includegraphics[width=\textwidth]{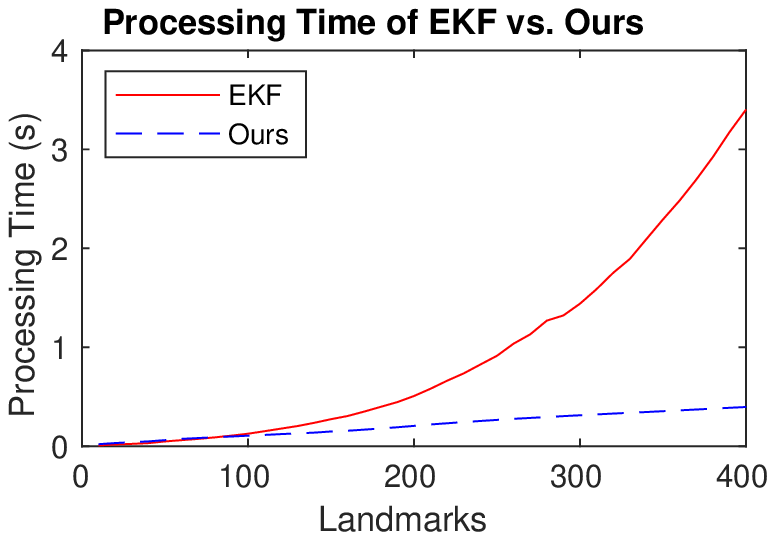}
    \caption{Mean computation time of EKF and our observer on 10-400 landmarks over 500 trials.}
    \label{fig:sim_time_comparison}
  \end{subfigure}
  \caption{Results of the simulation experiments comparing an EKF with the proposed observer.}
\end{figure}

Figure \ref{fig:sim_rmse_comparison} compares the statistics of the RMSE of the EKF and our observer for $n=50$ landmarks after 100 s over 500 trials.
While the EKF has a slightly lower mean RMSE, there are also more outliers due to linearisations errors.
Figure \ref{fig:sim_time_comparison} shows the mean computation time of the EKF and our observer for an increasing value of $n$ between 10 and 400 landmarks over 500 trials per number of landmarks.
While the processing time depends on the implementation of the EKF and of our observer, the figure clearly illustrates the quadratic complexity of the EKF and the linear complexity of our observer.

\section{Conclusion}
\label{sec:conclusion}

This paper proposes a new symmetry for visual SLAM and VIO problems.
This geometry is exploited to develop a visual SLAM observer and provide an almost global asymptotic and local exponential stability proof.
The authors believe that the inherent simplicity and robustness of the proposed approach makes it useful as a tool for embedded robotics applications.

\section*{Acknowledgement}

This research was supported by the Australian Research Council
through the ``Australian Centre of Excellence for Robotic Vision'' CE140100016.

\bibliographystyle{plainnat}
\bibliography{references}

\end{document}

%% file: preamble_compatible.tex
\usepackage{xcolor}  

\usepackage{mdframed} 
\usepackage{lipsum}  
\newmdenv[%
leftmargin = 0,%
rightmargin = 0,%
backgroundcolor = green,%
innertopmargin = \topskip,%
splittopskip = \topskip,%
skipabove =\baselineskip,%
skipbelow = \baselineskip]
{RMnotes}
\usepackage{comment}
\specialcomment{notes}{\begin{RMnotes}}{\end{RMnotes}}

\usepackage{MnSymbol}
\usepackage{mathdots}
\usepackage{mathtools} 

\usepackage{tensor}

\usepackage{amsfonts}

\usepackage{graphicx}
\usepackage{ifpdf}
\ifpdf
\usepackage{epstopdf}
\PrependGraphicsExtensions{.svg}
\fi

\newtheorem{theorem}{Theorem}[section]
\newtheorem{lemma}[theorem]{Lemma}

\newtheorem{proposition}[theorem]{Proposition}

\newtheorem{remark}[theorem]{Remark} 

%% file: notation_defs.tex
\newcommand{\R}{\mathbb{R}}

\newcommand{\SO}{\mathbf{SO}}

\newcommand{\SE}{\mathbf{SE}}

\newcommand{\so}{\mathfrak{so}}
\newcommand{\se}{\mathfrak{se}}

\newcommand{\gothr}{\mathfrak{r}}

\newcommand{\Sph}{\mathrm{S}}

\newcommand{\calM}{\mathcal{M}}
\newcommand{\calN}{\mathcal{N}}

\newcommand{\totT}{\mathcal{T}}

\newcommand{\vecV}{\mathbb{V}}

\DeclareMathOperator{\Ad}{Ad}

\newcommand{\id}{\mathrm{id}} 

\newcommand{\Lyap}{\mathcal{L}} 

\def\frameZero{\mbox{$\{0\}$}}

\newcommand{\mr}{\mathring} 

\usepackage{accents}
\makeatletter
\newcommand{\scirc}{%
    \hbox{\fontfamily{lmr}\fontsize{.4\dimexpr(\f@size pt)}{0}\selectfont{\raisebox{-0.34ex}[0ex][-0.34ex]{$\circ$}}}}

\makeatother

\mathchardef\mhyphen="2D

\newcommand{\td}{\mathrm{d}}
\newcommand{\tD}{\mathrm{D}}

\newcommand{\ddt}{\frac{\td}{\td t}}


%% file: Mahony_VIO_rss.bbl
\begin{thebibliography}{24}
\providecommand{\natexlab}[1]{#1}
\providecommand{\url}[1]{\texttt{#1}}
\expandafter\ifx\csname urlstyle\endcsname\relax
  \providecommand{\doi}[1]{doi: #1}\else
  \providecommand{\doi}{doi: \begingroup \urlstyle{rm}\Url}\fi

\bibitem[Barrau and Bonnabel(2016)]{2016_Barrau_arxive}
Axel Barrau and Silvere Bonnabel.
\newblock An ekf-slam algorithm with consistency properties.
\newblock \emph{arXiv:1510.06263}, 2016.
\newblock URL \url{https://arxiv.org/abs/1510.06263v3}.
\newblock arXiv:1510.06263.

\bibitem[Barrau and Bonnabel(2017)]{2017_Barrau_tac}
Axel Barrau and Silvère Bonnabel.
\newblock The invariant extended kalman filter as a stable observer.
\newblock \emph{IEEE Transactions on Automatic Control}, 62\penalty0
  (4):\penalty0 1797--1812, 2017.
\newblock \doi{DOI: 10.1109/TAC.2016.2594085}.

\bibitem[Bloesch et~al.(2015)Bloesch, Omari, Hutter, and
  Siegwart]{2015_Bloesch_iros}
M.~Bloesch, S.~Omari, M.~Hutter, and R.~Siegwart.
\newblock Robust visual inertial odometry using a direct ekf-based approach.
\newblock In \emph{International Conference on Intelligent Robotics (IROS)},
  2015.

\bibitem[Bonin-Font et~al.(2008)Bonin-Font, Ortiz, and
  Oliver]{2008_Bonin-Font_JIRS}
Francisco Bonin-Font, Alberto Ortiz, and Gabriel Oliver.
\newblock Visual navigation for mobile robots: A survey.
\newblock \emph{Journal of Intelligent Robotic Systems}, 53:\penalty0 263--296,
  2008.
\newblock \doi{DOI 10.1007/s10846-008-9235-4}.

\bibitem[Cadena et~al.(2016)Cadena, Carlone, Carrillo, Latif, Scaramuzza,
  Neira, Reid, and Leonard]{2016_Cadena_TRO}
Cesar Cadena, Luca Carlone, Henry Carrillo, Yasir Latif, Davide Scaramuzza,
  JosÂ´e Neira, {Ian D.} Reid, and {John J.} Leonard.
\newblock Past, present, and future of simultaneous localization and mapping:
  Towards the robust-perception age.
\newblock \emph{IEEE Transactions on Robotics}, 32\penalty0 (6):\penalty0
  1309--1332, December 2016.

\bibitem[Delmerico and Scaramuzza(2018)]{2018_Delmerico_icra}
J.~Delmerico and D.~Scaramuzza.
\newblock A benchmark comparison of monocular visual-inertial odometry
  algorithms for flying robots.
\newblock In \emph{IEEE International Conference on Robotics and Automation
  (ICRA)}, 2018.

\bibitem[Dissanayake et~al.(2011)Dissanayake, Huang, Wang, and
  Ranasinghe.]{2011_Dissanayake_ICIIS}
G.~Dissanayake, S.~Huang, Z.~Wang, and R.~Ranasinghe.
\newblock A review of recent developments in simultaneous localization and
  mapping.
\newblock In \emph{International Conference on Industrial and Information
  Systems}, pages 477--482, 2011.

\bibitem[Faessler et~al.(2016)Faessler, Fontana, Forster, Mueggler, Pizzoli,
  and Scaramuzza]{2016_Faessler_JFR}
M.~Faessler, F.~Fontana, C.~Forster, E.~Mueggler, M.~Pizzoli, and
  D.~Scaramuzza.
\newblock Autonomous, vision-based flight and live dense 3d mapping with a
  quadrotor {MAV}.
\newblock \emph{Journal of Field Robotics}, 33\penalty0 (4):\penalty0 431--450,
  2016.

\bibitem[Forster et~al.(2017{\natexlab{a}})Forster, Carlone, Dellaert, and
  Scaramuzza]{2017_ForsterDelaert_tro}
C.~Forster, L.~Carlone, F.~Dellaert, and D.~Scaramuzza.
\newblock On-manifold preintegration for real-time visual-inertial odometry.
\newblock \emph{IEEE Transactions on Robotics}, 33\penalty0 (1):\penalty0
  1--21, 2017{\natexlab{a}}.

\bibitem[Forster et~al.(2017{\natexlab{b}})Forster, Zhang, Gassner, Werlberger,
  and Scaramuzza]{2017_ForsterScaramuzza_tro}
C.~Forster, Z.~Zhang, M.~Gassner, M.~Werlberger, and D.~Scaramuzza.
\newblock Svo: Semidirect visual odometry for monocular and multicamera
  systems.
\newblock \emph{IEEE Transactions on Robotics}, 33\penalty0 (2):\penalty0
  249--265, 2017{\natexlab{b}}.

\bibitem[Guerreiro et~al.(2013)Guerreiro, Batista, Silvestre, and
  Oliveira]{2013_Guerreiro_TRO}
{Bruno J. N.} Guerreiro, Pedro Batista, Carlos Silvestre, and Paulo Oliveira.
\newblock Globally asymptotically stable sensor-based simultaneous localization
  and mapping.
\newblock \emph{IEEE Transactions on Robotics}, 29\penalty0 (6):\penalty0
  1380--1395, 2013.
\newblock \doi{DOI: 10.1109/TRO.2013.2273838}.

\bibitem[Hamel and Samson(2016)]{2016_Hamel_cdc}
T.~Hamel and C.~Samson.
\newblock Riccati observers for position and velocity bias estimation from
  direction measurements.
\newblock In \emph{2016 IEEE 55th Conference on Decision and Control (CDC)},
  pages 2047--2053, Dec 2016.
\newblock \doi{10.1109/CDC.2016.7798565}.

\bibitem[Kaess et~al.(2012)Kaess, Johannsson, Roberts, Ila, Leonard, and
  Dellaert]{2012_Kaess_IJRR}
Michael Kaess, Hordur Johannsson, Richard Roberts, Viorela Ila, J.J. Leonard,
  and Frank Dellaert.
\newblock isam2: Incremental smoothing and mapping using the bayes tree.
\newblock \emph{The International Journal of Robotics Research}, 31:\penalty0
  216--235, 2012.
\newblock \doi{doi.org/10.1177/0278364911430419}.
\newblock URL \url{https://doi.org/10.1177/0278364911430419}.

\bibitem[Kanatani and Morris(2001)]{2001_Kanatani_TIT}
K.~Kanatani and D.D. Morris.
\newblock Gauges and gauge transformations for uncertainty description of
  geometric structure with indeterminacy.
\newblock \emph{IEEE Transactions on Information Theory}, 47\penalty0
  (5):\penalty0 2017--2028, 2001.

\bibitem[Leutenegger et~al.(2015)Leutenegger, Lynen, Bosse, Siegwart, and
  Furgale]{2015_Leutenegger_ijrr}
S.~Leutenegger, S.~Lynen, M.~Bosse, R.~Siegwart, and P.~Furgale.
\newblock Keyframe-based visual-inertial slam using nonlinear optimization.
\newblock \emph{International Jouranl on Robotic Research}, 34\penalty0
  (3):\penalty0 314–--334, 2015.

\bibitem[Louren{\c{c}}o et~al.(2016)Louren{\c{c}}o, Guerreiro, Batista,
  Oliveira, and Silvestre]{2016_LouGueBatOliSil}
Pedro Louren{\c{c}}o, Bruno Guerreiro, Pedro Batista, Paulo Oliveira, and
  Carlos Silvestre.
\newblock Simultaneous localization and mapping for aerial vehicles: a 3-d
  sensor-based gas filter.
\newblock \emph{Autonomous Robots}, 40\penalty0 (5):\penalty0 881--902, 2016.

\bibitem[Lynen et~al.(2013)Lynen, Achtelik, Weiss, Chli, and
  Siegwart]{2014_Lynen_iros}
S.~Lynen, M.~Achtelik, S.~Weiss, M.~Chli, and R.~Siegwart.
\newblock A robust and modular multi-sensor fusion approach applied to mav
  navigation.
\newblock In \emph{Proceedings of International Conference on Intelligent
  Robotics (IROS)}, 2013.

\bibitem[Mahony and Hamel(2017)]{2017_Mahony_cdc}
Robert Mahony and Tarek Hamel.
\newblock A geometric nonlinear observer for simultaneous localisation and
  mapping.
\newblock In \emph{Conference on Decision and Control}, page 6 pages,
  Melbourne, December 2017.

\bibitem[Mahony et~al.(2013)Mahony, Trumpf, and Hamel]{RM_2013_Mahony_nolcos}
Robert Mahony, Jochen Trumpf, and Tarek Hamel.
\newblock Observers for kinematic systems with symmetry.
\newblock In \emph{Proceedings of 9th IFAC Symposium on Nonlinear Control
  Systems (NOLCOS)}, page 17 pages, 2013.
\newblock Plenary paper.

\bibitem[Mourikis and Roumeliotis(2007)]{2007_Mourikis_icra}
A.I. Mourikis and S.I. Roumeliotis.
\newblock A multi-state constraint {Kalman} filter for vision-aided inertial
  navigation.
\newblock In \emph{Procedings of the International Conference on Robotics and
  Automation (ICRA)}, 2007.

\bibitem[Qin et~al.(2017)Qin, Li, and Shen]{2017_Qin_arxive}
T.~Qin, P.~Li, and S.~Shen.
\newblock Vins-mono: A robust and versatile monocular visual-inertial state
  estimator.
\newblock arXiv:1708.03852, 2017.

\bibitem[Stachniss et~al.(2016)Stachniss, Thrun, and
  Leonard]{2016_Stachniss_Handbook}
C.~Stachniss, S.~Thrun, and J.J. Leonard.
\newblock \emph{Simultaneous Localization and Mapping. In B. Siciliano},
  chapter~46, pages 1153–--1176.
\newblock Springer, 2nd edition, 2016.

\bibitem[van Goor et~al.(2019)van Goor, Mahony, Hamel, and
  Trumpf]{2019_vangoor_cdc_vslam}
Pieter van Goor, Robert~E. Mahony, Tarek Hamel, and Jochen Trumpf.
\newblock An equivariant observer design for visual localisation and mapping.
\newblock \emph{CoRR}, abs/1904.02452, 2019.
\newblock URL \url{http://arxiv.org/abs/1904.02452}.

\bibitem[Zhang et~al.(2017)Zhang, Wu, Song, Huang, and
  Dissanayake]{2017_Zhang_ral}
Teng Zhang, Kanzhi Wu, Jingwei Song, Shoudong Huang, and Gamini Dissanayake.
\newblock Convergence and consistency analysis for a 3-d invariant-{EKF}
  {SLAM}.
\newblock \emph{IEEE Robotics and Automation Letters}, 2\penalty0 (2):\penalty0
  733--740, 2017.

\end{thebibliography}
